\documentclass[sigconf,nonacm]{acmart}

\usepackage{amsmath}
\usepackage{algorithm,algorithmic}
\usepackage{booktabs}
\usepackage{listings,xcolor}
\usepackage{enumitem}
\usepackage{stmaryrd}

\newtheorem{theorem}{Theorem}[section]

\newtheorem{corollary}[theorem]{Corollary}

\newtheorem{definition}[theorem]{Definition}

\newcommand{\MM}{\mathcal{MM}}
\newcommand{\Lib}{\mathcal{W}}
\newcommand{\Goal}{\mathcal{G}}
\newcommand{\phCOM}{\textsc{Compute}}
\newcommand{\phREQ}{\textsc{Require}}
\newcommand{\phEXE}{\textsc{Execute}}
\newcommand{\phCAL}{\textsc{Call}}

\AtBeginDocument{%
  \providecommand\BibTeX{{\normalfont B\kern-0.5em{\scshape i\kern-0.25em b}\kern-0.8em\TeX}}}

\copyrightyear{2026}
\acmYear{2026}
\acmDOI{}

\begin{document}
\microtypesetup{expansion=false,protrusion=false}
\tolerance=1500
\hbadness=1500
\emergencystretch=2em

\title{Context: Proactive Goal-Directed Intelligence\\
via Composable Sandboxed Programs,\\
Declarative Wiring, and Structured Interaction}

\author{Gregory Magarshak}
\email{gmagarshak@faculty.ienyc.edu}

\affiliation{%
  \institution{Qbix, Inc.\ \& Intercoin, Inc.}
  \city{New York}
  \country{USA}
}
\additionalaffiliation{%
  \institution{IE University NYC}
  \country{USA}
}

\begin{abstract}
We present \textbf{Context}, the intelligence layer of the Magarshak Architecture,
which replaces reactive query-response chatbots with proactive goal-directed agents
that advance shared tasks without waiting for user prompts.
The architecture rests on three mutually reinforcing mechanisms.
\emph{Write-time context assembly} precomputes enriched typed attributes via Groker
agents~\cite{magarshak2026grokers}, assembling interaction context from a
maintained denormalization index as a deterministic pure function of graph state;
context blocks are byte-identical across turns between semantic changes, enabling
near-100\% KV-cache reuse.
\emph{Composable sandboxed wisdom programs} form a governed library of
LM-generated imperative programs declaratively wired to goal types via typed stream
relations, composed via phase ordering, evolved through fitness-based selection,
and executed at interaction time without further LM calls.
\emph{Proactive goal stream state machines} drive conversations toward terminal
states by inspecting graph state and emitting structured interaction
content---option arrays, governance affordances, clarification prompts---without
awaiting user input.
We prove six formal results: (1)~the \emph{Context Stability Theorem}, bounding
per-turn LM cost as a function of semantic change rate; (2)~the \emph{Program
Composition Correctness Theorem}, establishing that phase-ordered wisdom program
pipelines preserve individual correctness; (3)~the \emph{Declarative Wiring
Soundness Theorem}, establishing complete and sound event routing; (4)~the
\emph{Proactive Dominance Theorem}, proving that proactive agents weakly dominate
reactive agents on expected turns-to-terminal-state; (5)~the \emph{Coordination
Overhead Elimination Theorem} and \emph{Quality Preservation Theorem}, establishing
that proactive agents are Pareto improvements in multi-participant goal chats---fewer
turns, equal or higher artifact quality; and (6)~the \emph{Cross-Platform Vote
Consistency Theorem}, guaranteeing governance consistency across Telegram, email,
web, and mobile.
The architecture is implemented in the open-source Qbix / Safebox / Safebots stack.
\end{abstract}

\ccsdesc[500]{Human-centered computing~Collaborative and social computing systems and tools}
\ccsdesc[300]{Computing methodologies~Natural language processing}
\ccsdesc[300]{Computing methodologies~Knowledge representation and reasoning}

\keywords{proactive dialogue, goal-directed agents, wisdom library, sandboxed programs,
organizational efficiency, knowledge graphs, KV-cache optimization, cross-platform governance}

\maketitle

\section{Introduction}
\label{sec:intro}

Every deployed conversational AI system today is fundamentally reactive.
It receives a message and emits a response.
For genuinely open-ended questions this is appropriate; for the large majority of
goal-directed interactions---building a software capability, reviewing a document,
resolving a support ticket, governing a shared artifact---the required next action
is often \emph{determinable from prior graph state without any new user input at all}.

Reactive agents in multi-participant goal chats impose a structural inefficiency
with no natural corrective: a large fraction of turns are \emph{coordination
turns}---turns spent establishing the current state, identifying blockers, and
deciding who acts next.
These turns produce no progress toward the terminal state; they are pure overhead.
We formalize this and prove that proactive agents eliminate it structurally.

Three specific primitives are absent from current architectures:
\emph{write-time context}~\cite{magarshak2026grokers},
\emph{composable sandboxed programs}, and
\emph{proactive state machine intelligence}.
This paper introduces all three as an integrated architecture and proves formal
properties of each.

\subsection*{Positioning in the Magarshak Architecture}

This paper is the third in a series.
Magarshak~\cite{magarshak2026mm} introduces the Magarshak Machine ($\MM$), the
SPACER substrate: append-only streams, policy governance, five-phase action
execution ($\phCOM\to\phREQ\to\phEXE$), and the bidirectional relation index.
Magarshak~\cite{magarshak2026grokers} introduces Grokers, the write-time
comprehension layer: bottom-up inductive enrichment, the Byte-Identity Theorem,
and Accumulation Monotonicity.
The present paper introduces the \textbf{Context} intelligence layer: proactive
goal-directed agents, wisdom library composition and declarative wiring,
cross-platform rendering, and organizational efficiency theorems.
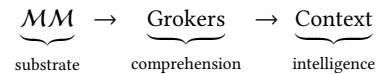
\[
  \underbrace{\MM}_{\text{substrate}}
  \;\to\;
  \underbrace{\text{Grokers}}_{\text{comprehension}}
  \;\to\;
  \underbrace{\text{Context}}_{\text{intelligence}}
\]

\subsection*{Contributions}
\begin{enumerate}[leftmargin=*,label=(\arabic*)]
\item Formal model of the Context architecture (\S\ref{sec:model}).
\item Wisdom library composition algebra with correctness-preservation theorem
      (\S\ref{sec:wisdom}).
\item Declarative wiring model with soundness theorem (\S\ref{sec:wiring}).
\item Proactive state machine framework with Proactive Dominance Theorem
      (\S\ref{sec:proactive}).
\item Organizational efficiency theorems: Coordination Overhead Elimination and
      Quality Preservation (\S\ref{sec:org}).
\item Cross-platform vote consistency theorem (\S\ref{sec:platform}).
\item Dual-traversal Hierarchy-Cache Correspondence Theorem (\S\ref{sec:dual}).
\end{enumerate}

\section{Related Work}
\label{sec:related}

\textbf{Task-oriented dialogue.}
POMDP-based systems~\cite{young2013pomdp} and neural dialogue
models~\cite{henderson2015machine,budzianowski2019hello} model conversation as
belief-state tracking followed by act selection.
They do not address write-time context precomputation, governed imperative programs,
or organizational efficiency in multi-participant settings.
Context inherits the state machine framing and extends it with deterministic
graph-derived context, a growing wisdom library replacing learned policies, and
proactive advancement without user prompts.

\textbf{Proactive dialogue.}
Proactive systems have been studied in
recommendation~\cite{deng2023survey} and knowledge-grounded
conversation~\cite{wu2019proactive}.
These define proactivity as topic introduction or recommendation surfacing.
Context defines proactivity at the structural level: a bot acts when a state machine
advancement condition is satisfied by graph state, correct by construction rather
than approximately correct by training.

\textbf{Multi-agent collaboration.}
AutoGen~\cite{wu2024autogen}, MetaGPT~\cite{hong2023metagpt}, and
CAMEL~\cite{li2023camel} focus on agent-to-agent coordination.
Context is concerned with the organizational efficiency of \emph{human-AI-human}
interactions: the fraction of turns spent on coordination versus productive progress,
and how the wisdom library and proactive state machine reduce that fraction toward zero.

\textbf{RAG.}
RAG systems~\cite{lewis2020rag,edge2024graphrag} address knowledge access at query
time.
Context replaces query-time retrieval with write-time enrichment and deterministic
graph reads; the key formal distinction is the byte-identity property of
deterministically assembled context blocks, enabling near-100\% KV-cache reuse
not achievable by any RAG system.

\textbf{LM-generated programs.}
Using LMs to generate executable code is well-studied~\cite{chen2021codex,austin2021program}.
The wisdom library adds: governance (programs reviewed before activation);
fitness-based evolutionary selection; the read-only sandbox contract; and
declarative phase assignment via typed stream relations.

\textbf{Organizational AI.}
Empirical studies~\cite{noy2023experimental,brynjolfsson2023generative} measure AI
productivity on individual tasks but do not formalize multi-participant coordination
overhead or prove bounds on proactivity's effect.
Section~\ref{sec:org} provides the first formal model and theorems on this question.

\section{Formal Model}
\label{sec:model}

We build on the typed stream graph of~\cite{magarshak2026grokers} and the SPACER
operational semantics of~\cite{magarshak2026mm}.

\begin{definition}[Goal Stream]
\label{def:goal}
\sloppy
A \emph{goal stream} $\Goal=(T,Q,q_0,F,\delta,\Lambda,\Lib_\Goal,\Pi_\Goal)$
is a tuple where:
$T$ is a stream type;
$Q$ a finite state set;
$q_0\in Q$ the initial state;
$F\subseteq Q$ the terminal states;
$\delta:Q\times\Sigma_{\mathit{attr}}\to Q$ the transition function, where
$\Sigma_{\mathit{attr}}=\{(\mathit{field},\mathit{op},\mathit{val})\mid
\mathit{field}\in K,\;\mathit{op}\in\{=,{\neq},{>},{\geq},{<},{\leq}\},
\;\mathit{val}\in A\}$
is the set of typed attribute conditions;
$\Lambda:Q\to\mathcal{P}(\mathit{InputMode})$ the per-state input mode function;
$\Lib_\Goal$ the wisdom library (Def.~\ref{def:library}); and
$\Pi_\Goal$ the proactive advancement conditions (Def.~\ref{def:proactive}).
\end{definition}

\begin{definition}[Wisdom Library]
\label{def:library}
A \emph{wisdom library} $\Lib=\{p_i\}$ is a finite set of \emph{wisdom programs},
each $p_i=(n_i,\phi_i,I_i,O_i,f_i,\ell_i)$ where
$n_i$ is a name; $\phi_i\in\Phi$ an execution phase;
$I_i,O_i$ input/output schemas; $f_i\in[0,1]$ fitness; and
$\ell_i$ the imperative program text whose denotation
$\llbracket\ell_i\rrbracket:I_i\to O_i$ is the function computed by $\ell_i$.
We write $p_i(x)$ as shorthand for $\llbracket\ell_i\rrbracket(x)$.
Programs execute in the Safebox sandbox: reads from a pre-loaded immutable input
(no live DB queries, corresponding to SPACER's $\phCOM$); writes only via proposal
accumulation (no direct writes, corresponding to $\phREQ$);
time $\leq50\,$ms; memory $\leq64\,$MB; no network except named Protocols.
\end{definition}

\begin{definition}[Execution Phase Algebra]
\label{def:phases}
Phases $\Phi=\{\mathsf{pre},\mathsf{ctx},\mathsf{post},\mathsf{auto},
\mathsf{render},\mathsf{rel},\mathsf{agg},\mathsf{idx}\}$ are partially ordered
by data-flow: $\mathsf{pre}\prec\mathsf{ctx}\prec\mathsf{agg}\prec
\mathsf{post}\prec\mathsf{render}$;
$\mathsf{rel}\prec\mathsf{agg}$;
$\mathsf{post}\prec\mathsf{auto}$;
$\mathsf{idx}$ independent.
Programs at incomparable phases (e.g., $\mathsf{rel}$ and $\mathsf{ctx}$) have no
data-flow dependency and may execute in any order or concurrently.
Sequential composition $p\mathbin{;}q$ (with $\phi(p)\prec\phi(q)$) runs $p$ then
$q$ on the merged input $x\oplus p(x)$.
Parallel composition $p\mathbin{\|}q$ (same phase or incomparable phases, disjoint
output keys) produces merged output $p(x)\cup q(x)$.
\end{definition}

\begin{definition}[Context Block Hierarchy]
\label{def:context}
$\mathcal{C}(\Goal,v)=B_{\mathsf{perm}}(\Goal)\cdot B_{\mathsf{sess}}(v)\cdot
B_{\mathsf{cold}}(v)\cdot B_{\mathsf{dyn}}(\Goal,v,t)$ where:
$B_{\mathsf{perm}}$: goal system prompt, stable for the goal type's lifetime (KV BP1);
$B_{\mathsf{sess}}$: $C(v)$ of~\cite{magarshak2026grokers}, stable between semantic
changes (KV BP2);
$B_{\mathsf{cold}}$: multi-level summary tree, included only for \emph{cold sessions}
(sessions resuming after the KV-cache TTL has expired, typically $\geq5$ minutes
of inactivity, requiring the cached prefix to be re-uploaded);
$B_{\mathsf{dyn}}$: wisdom-program-selected context, per-turn.
\end{definition}

\begin{theorem}[Context Stability]
\label{thm:stability}
The expected per-turn LM input cost is:
\[
\bar{C}_{\mathit{turn}}=0.1(k_{\mathsf{perm}}+k_{\mathsf{sess}})+
k_{\mathsf{cold}}\cdot\mathbf{1}[\text{cold}]+k_{\mathsf{dyn}}
\]
where stable blocks cost 10\% of full price by the Byte-Identity
Theorem~\cite{magarshak2026grokers}.
As $T_c(v)/T_t\to\infty$, the cached component approaches
$0.1(k_{\mathsf{perm}}+k_{\mathsf{sess}})$---a $10\times$ reduction relative to
paying full price on the stable prefix.
\end{theorem}
\begin{proof}
By the Byte-Identity Theorem~\cite{magarshak2026grokers}, $B_{\mathsf{perm}}$ and
$B_{\mathsf{sess}}$ are byte-identical across turns between semantic changes;
this makes them eligible for KV-cache reuse.
The KV-cache hit probability on $B_{\mathsf{sess}}$ equals $1-T_t/T_c(v)$
under a Poisson change process, approaching 1 as $T_c(v)\gg T_t$.
Cache-hit tokens cost 10\% of full input price~\cite{anthropic2024caching},
independently of the byte-identity property.
The cold block is amortized over the session; the dynamic block is always charged in full.
\end{proof}

\section{Wisdom Library: Composition and Wiring}
\label{sec:wisdom}

\subsection{Program Composition Correctness}

\begin{definition}[Phase-Correct Library]
$\Lib$ is \emph{phase-correct} for $\Goal$ if:
(a)~for every $p,q\in\Lib$ with $\phi(p)\prec\phi(q)$ and any input $x\in I_p$,
the merged input $x\oplus p(x)\in I_q$, where $\oplus$ denotes schema-compatible
record merge (prior fields are preserved; $q$'s required fields are provided by
$x$ or $p(x)$);
(b)~the aggregation program accepts the union of all outputs from programs with
$\phi\prec\mathsf{agg}$; and
(c)~rendering programs accept the aggregation output.
\end{definition}

\begin{theorem}[Program Composition Correctness]
\label{thm:composition}
Let $\Lib$ be phase-correct and let $p_1,\ldots,p_n\in\Lib$ be in valid phase order.
If each $p_i$ is individually correct ($\forall x\in I_i$, $p_i(x)\in O_i$), then
the composed pipeline $P=p_1\mathbin{;}{\cdots}\mathbin{;}p_n$ is correct.
\end{theorem}
\begin{proof}
Induction on $n$.
\emph{Base}: $P=p_1$ is correct by hypothesis.
\emph{Step}: $P_{n-1}$ is correct by inductive hypothesis, producing $y_{n-1}\in O_{n-1}$.
Phase-correctness (condition a) gives $x_n=x\oplus y_{n-1}\in I_n$.
Individual correctness of $p_n$ gives $p_n(x_n)\in O_n$.
\end{proof}

\begin{corollary}[Modular Extension]
If $\Lib$ is phase-correct and a new program $p$ satisfies the phase-correctness
conditions at addition time, then $\Lib\cup\{p\}$ is phase-correct and
$P_{\Lib\cup\{p\}}$ is correct.
\end{corollary}

Programs can be added to the wisdom library without breaking existing correctness,
provided the addition protocol verifies phase-correctness.
Combined with Accumulation Monotonicity~\cite{magarshak2026grokers},
the library grows in both coverage and correctness simultaneously.

\subsection{Declarative Wiring and Policy Graphs}
\label{sec:wiring}

\begin{definition}[Policy Graph]
\sloppy
The \emph{policy graph} of publisher $\rho$ is a directed graph
$\mathcal{P}_\rho=(H,E_{\mathcal{P}})$ where nodes $H=\{h_i\}$ are event handlers
(policies, post-generation hooks, workflow steps, platform adapters), edges
$(h_i,h_j)$ denote output-triggers-input, and each $h_i$ is annotated with:
stream type pattern $\tau_i$; event type $e_i\in\{\mathtt{stream.created},\allowbreak\mathtt{stream.updated},
\allowbreak\mathtt{vote.cast},\allowbreak\mathtt{job.completed},\ldots\}$;
and condition predicate
$\gamma_i:\mathit{Attrs}\to\mathbb{B}$.
Handlers are declared as typed stream nodes related to target streams via
\texttt{Safebox/subscribes} relations---no imperative registration code.
\end{definition}

\begin{theorem}[Declarative Wiring Soundness]
\label{thm:wiring}
Assume SPACER's event queue provides at-most-once delivery per subscription
(each $(v,e,t)$ event is enqueued exactly once per matching handler by
Rule~\textsc{Trigger}~\cite{magarshak2026mm}).
Then for every stream event $(v,e,t)$ within publisher $\rho$'s scope:
(i)~\emph{(Completeness)} every handler $h_i$ with $\tau(v)\models\tau_i$, $e=e_i$,
and $\gamma_i(\alpha(v))=\mathtt{true}$ is triggered exactly once;
(ii)~\emph{(Soundness)} no handler with $\tau(v)\not\models\tau_j$, $e\neq e_j$,
or $\gamma_j(\alpha(v))=\mathtt{false}$ is triggered.
\end{theorem}
\begin{proof}
SPACER Rule~\textsc{Trigger}~\cite{magarshak2026mm} enqueues exactly the action
invocations subscribed to $v$ filtered by message type.
\texttt{Safebox/subscribes} relations define the subscription set; $\gamma_i$ is
evaluated in $\phCOM$ before any side effects---handlers for which $\gamma_i=\mathtt{false}$
exit before $\phEXE$, producing no observable effect.
Completeness: any $(v,e,t)$ with matching $(\tau_i,e_i,\gamma_i)$ has $h_i$ in
the subscription set; Rule~\textsc{Trigger} fires exactly one invocation per
subscription; $\gamma_i$ evaluates true so $\phEXE$ proceeds.
Soundness: non-matching handlers are not in the subscription set (by
\texttt{Safebox/subscribes} construction); $\gamma_j=\mathtt{false}$ handlers exit
before $\phEXE$.
\end{proof}

\begin{corollary}[Publisher Inheritance Soundness]
If $\rho'$ inherits from $\rho$ via \texttt{Safebox/inherits} and overrides handler
$h_i$ with $h_i'$, Theorem~\ref{thm:wiring} holds for $\rho'$ with $h_i$
replaced by $h_i'$.
\end{corollary}

\section{Proactive State Machine Intelligence}
\label{sec:proactive}

\begin{definition}[State Machine Position]
Let $H_T(v)=\langle e_1, e_2, \ldots, e_k\rangle$ be the sequence of attribute
update events on stream $v$ from time $0$ to time $T$, ordered by timestamp,
where each $e_i \in \Sigma_{\mathit{attr}}$ is an attribute condition that became
satisfied at event $i$.
The \emph{position} of goal stream instance $v$ at time $T$ is
$q_T(v)=\delta^*(q_0, H_T(v))$
where $\delta^*: Q \times \Sigma_{\mathit{attr}}^* \to Q$ is the standard extension
of $\delta$ to sequences, defined by $\delta^*(q,\epsilon)=q$ and
$\delta^*(q, \sigma \cdot e) = \delta(\delta^*(q, \sigma), e)$.
\end{definition}

\begin{definition}[Proactive Advancement Condition]
\label{def:proactive}
A \emph{proactive advancement condition} $(q,\gamma,\mu)$: $q\in Q$ is the target
state; $\gamma:\mathit{Attrs}\times G^T\to\mathbb{B}$ a condition over stream
attributes and induced graph; $\mu:\mathit{Attrs}\times G^T\to\mathit{Messages}$
a \emph{proactive message generator} producing structured content (option arrays,
governance affordances, clarification prompts) from graph state alone.
When $q_T(v)=q$ and $\gamma(\alpha_T(v),G^T)=\mathtt{true}$, the system emits
$\mu(\alpha_T(v),G^T)$ without awaiting user input.
\end{definition}

\begin{definition}[Reactive Agent]
\label{def:reactive}
An agent policy $\pi_R$ is \emph{reactive} if it emits output only in response to
user messages: for all times $T$, no output is emitted in any interval
$[T, T+\epsilon]$ in which no user message arrives.
\end{definition}

\begin{definition}[Proactive Agent]
\label{def:proactiveagent}
An agent policy $\pi_P$ is \emph{proactive} if it emits output whenever either
(a) a user message arrives, or (b) a proactive advancement condition
$(q,\gamma,\mu)\in\Pi_\Goal$ is satisfied: $q_T(v)=q$ and
$\gamma(\alpha_T(v),G^T)=\mathtt{true}$.
\end{definition}

\begin{definition}[Turns-to-Terminal]
$\mathcal{N}^\pi(v)$: the number of interaction turns from $q_0$ until
$q_T(v)\in F$ under agent policy $\pi$.
\end{definition}

\begin{theorem}[Proactive Dominance]
\label{thm:proactive}
Let $\pi_R$ (reactive) and $\pi_P$ (proactive) be agents for the same goal type
$\Goal$, identical in LM response quality when they respond.
For any goal stream instance $v$:
$\mathbb{E}[\mathcal{N}^{\pi_P}(v)]\leq\mathbb{E}[\mathcal{N}^{\pi_R}(v)]$,
with equality if and only if $\Pi_\Goal=\emptyset$ or every
advancement condition in $\Pi_\Goal$ has $p_{\mathrm{user}}=1$
(i.e., every such condition is satisfied exclusively by user messages).
\end{theorem}
\begin{proof}
Each transition condition $c_i$ is satisfied either by (a) a user message, or (b)
a background event (job completion, vote threshold, dependency resolution, elapsed time).

Under $\pi_R$ (Definition~\ref{def:reactive}): the reactive agent emits output only
in response to user messages.
If $c_i$ is satisfied by a background event type (b) at time $t_b$ strictly between
two user messages at times $t_{\mathrm{prev}} < t_b < t_{\mathrm{next}}$, the
reactive agent cannot emit an advancement message until $t_{\mathrm{next}}$.
This introduces at least one turn that would not appear under $\pi_P$: the user
message at $t_{\mathrm{next}}$ (or a subsequent one) is needed to trigger
$\pi_R$'s response to the background event.
More precisely, the expected number of turns between $c_i$ becoming satisfiable and
$\pi_R$ acting on it is at least $1$, since the reactive agent must wait for a user
message.

Under $\pi_P$ (Definition~\ref{def:proactiveagent}): if $(q,\gamma,\mu)\in\Pi_\Goal$
covers $c_i$, then immediately when $\gamma(\alpha_T(v), G^T)$ becomes true at time
$t_b$, the system emits $\mu(\alpha_T(v), G^T)$ without any user message.
The turn cost is zero additional user-message turns.

Summing over all conditions covered by $\Pi_\Goal$, the expected total turns
$\mathbb{E}[\mathcal{N}^{\pi_P}]$ is reduced by at least $|\{i : c_i \text{ covered}\}|$
relative to $\mathbb{E}[\mathcal{N}^{\pi_R}]$.
For uncovered conditions, both policies behave identically.
Therefore $\mathbb{E}[\mathcal{N}^{\pi_P}]\leq\mathbb{E}[\mathcal{N}^{\pi_R}]$.
\end{proof}

\begin{corollary}[Proactive Savings Bound]
Let $n_P$ be the number of transition conditions covered by advancement conditions
and $p_{\mathrm{user}}$ the probability that a condition is satisfied
simultaneously by a user message.
The expected turn savings is $\Delta\mathcal{N}\geq n_P(1-p_{\mathrm{user}})$.
For background-event-dominated conditions, $p_{\mathrm{user}}\approx 0$ and
savings $\approx n_P$.
\end{corollary}

\section{Organizational Efficiency}
\label{sec:org}

\begin{definition}[Turn Classification]
A turn in a multi-participant goal chat is:
\emph{progress} (modifies graph state or advances the state machine);
\emph{coordination} (communicates state, identifies a blocker, assigns responsibility,
without modifying the graph);
\emph{governance} (casts a vote, records approval, registers a fork);
or \emph{exploratory} (belongs to a working branch thread).
The \emph{coordination overhead ratio} is
$\Omega=N_{\mathrm{coord}}/(N_{\mathrm{prog}}+N_{\mathrm{gov}})$.
\end{definition}

\begin{theorem}[Coordination Overhead Structural Decomposition]
\label{thm:coordination}
Let the coordination turns in any goal chat be partitioned into categories
$C_1, \ldots, C_k$ (state-inquiry, blocker-identification, responsibility-assignment,
vote-solicitation, etc.), with weight $w_i = N_{C_i}/N_{\mathrm{coord}} \geq 0$
(the fraction of total coordination turns of category $C_i$), so $\sum_i w_i = 1$.
For a proactive agent $\pi_P$ whose advancement conditions and wisdom programs
structurally preempt category $C_i$ with coverage $\mathrm{cov}(C_i)\in[0,1]$,
define $c_{\mathrm{elim}}=\sum_i w_i \cdot \mathrm{cov}(C_i)$.
Then in any goal chat instance:
$\Omega^{\pi_P}\leq\Omega^{\pi_R}\cdot(1-c_{\mathrm{elim}})$.
For a mature system with complete coverage of C1--C4, $c_{\mathrm{elim}}\to 1$
and $\Omega^{\pi_P}\to 0$.
\end{theorem}
\begin{proof}
Coordination turns fall into four identifiable categories:
C1~(state-inquiry: ``What state are we in?''),
C2~(blocker-identification: ``What's blocking us?''),
C3~(responsibility-assignment: ``Who needs to act?''), and
C4~(vote-solicitation: ``Who hasn't voted?'').
Under $\pi_P$:
C1 is preempted by the session-join proactive advancement condition emitting the
current state and missing transition requirements before any participant asks;
C2 is preempted by a $\mathsf{pre}$-phase wisdom program that detects blocking
conditions after every graph mutation;
C3 is preempted by deriving responsibility from the goal state machine's role
requirements and emitting targeted prompts;
C4 is preempted by the governance advancement condition emitting vote affordances
when the tally is short specific participants.
Each category $C_i$ preempted with coverage $\mathrm{coverage}(C_i)$ contributes
to $c_{\mathrm{elim}}=\sum_i\mathrm{coverage}(C_i)\cdot\mathrm{weight}(C_i)$.
For a mature library with full coverage of C1--C4, $c_{\mathrm{elim}}\to 1$
and $\Omega^{\pi_P}\to 0$.
\end{proof}

\begin{theorem}[Quality Preservation Under Proactivity]
\label{thm:quality}
Let $Q(v)$ be any quality metric that depends only on the content produced by
deliberation turns (turns with routing decision $\mathsf{pass}$ to the LM),
not on coordination turns.
Then $\mathbb{E}[Q^{\pi_P}(v)]\geq\mathbb{E}[Q^{\pi_R}(v)]$.
\end{theorem}
\begin{proof}
\emph{Weak inequality.}
By Theorem~\ref{thm:proactive}, $\pi_P$ saves turns corresponding to proactive
advancement conditions.
Each such condition fires only when $\gamma$ is satisfied by \emph{current graph
state}---not by LM generation.
Content emitted by $\mu$ is a deterministic function of $\alpha(v)$ and $N(v)$;
none of the saved turns contain LM-synthesized deliberation content.
All LM calls for turns requiring genuine synthesis (routing decision $\mathsf{pass}$)
are present in both $\pi_R$ and $\pi_P$ interactions.
Therefore the saved turns contribute zero expected quality to the terminal artifact,
and removing them does not decrease $\mathbb{E}[Q]$:
$\mathbb{E}[Q^{\pi_P}]\geq\mathbb{E}[Q^{\pi_R}]$.

\emph{Source of strict improvement (informally).}
Under $\pi_R$, blocking conditions (unconfigured dependencies, pending votes,
unresolved ambiguities) that arise from background events may go undetected
until a participant explicitly queries for them, potentially allowing artifacts
to reach terminal state with unresolved issues.
Under $\pi_P$, advancement conditions surface these structurally before any
terminal transition fires.
This provides an additional reason why strict inequality holds in practice
for goal types with non-trivial governance, though formalizing this requires
additional assumptions on the governance threshold model.
\end{proof}

Theorems~\ref{thm:coordination} and~\ref{thm:quality} together establish the
central organizational result: proactive agents are \emph{Pareto improvements}
over reactive agents---strictly faster to terminal state, equal or higher artifact
quality.
No speed-quality tradeoff is incurred.

\begin{corollary}[Thread Isolation Efficiency]
Exploratory conversations isolated to working branch threads cost
$\Delta C_{\mathrm{ctx}}=n_E\cdot\bar{\ell}_E\cdot k_t$
tokens per main-stream turn less than under reactive assistance (where $n_E$ is
the number of concurrent threads and $\bar{\ell}_E$ is the mean thread length
in turns of $k_t$ tokens each).
Under reactive assistance, exploratory messages appear in the main stream and are
included in its DYNAMIC context block at full token cost.
Under proactive assistance with thread isolation, thread messages are posted to
separate thread streams excluded from the main stream's SESSION block, contributing
zero tokens to subsequent main-stream LM calls.
\end{corollary}

\begin{corollary}[Total Efficiency Gain]
\label{cor:efficiency}
Treating the three efficiency mechanisms as acting on independent components of
per-goal-completion cost (coordination turns, LM-call frequency, and per-call
token cost respectively), the combined upper bound is:
$C^{\pi_P}\leq C^{\pi_R}\cdot(1-c_{\mathrm{elim}})\cdot(1-E(\Lib^{(t)}))\cdot
\tfrac{k_{\mathsf{dyn}}+0.1(k_{\mathsf{perm}}+k_{\mathsf{sess}})}
     {k_{\mathsf{dyn}}+k_{\mathsf{perm}}+k_{\mathsf{sess}}}$
where the three factors capture coordination overhead elimination (Theorem~\ref{thm:coordination}),
LM-call elimination (Grokers Accumulation Monotonicity~\cite{magarshak2026grokers}),
and KV-cache savings (Theorem~\ref{thm:stability}).
The independence assumption is an approximation: in practice,
$c_{\mathrm{elim}}$ and $E(\Lib^{(t)})$ interact since eliminated coordination turns
also eliminate wisdom library invocations; the product formula provides a
conservative upper bound under the approximation that the interactions are negligible
relative to the dominant cost reductions.
\end{corollary}

\section{Cross-Platform Rendering and Vote Consistency}
\label{sec:platform}

\begin{definition}[Options Array and Platform Adapters]
\sloppy
The \emph{options array} $\mathcal{O}(v,q)$ is produced by $\mathsf{post}$-phase
wisdom programs from $(\alpha(v),N(v),q)$.
Each option specifies type, label, and structured payload.
\emph{Platform adapters} are $\mathsf{render}$-phase wisdom programs with input
$I_{\mathsf{render}}=\{\mathit{options}: \mathcal{O}(v,q),\;
\mathit{session}:\{\mathit{platform},\mathit{user\_id},\mathit{locale}\}\}$ and output $O_{\mathsf{render}}$ a
platform-native message payload, that map $\mathcal{O}$ to platform-native representations:
Telegram $\mathtt{InlineKeyboardMarkup}$,
Apple $\mathtt{InteractiveMessage}$,
Google $\mathtt{SuggestionChip+RichCard}$,
Email $\mathtt{AMP+HTMLFallback}$,
Web $\mathtt{ChipBar}$.
Adapters are declaratively wired to goal streams via \texttt{Safebox/adapts}
relations---no imperative dispatch code.
\end{definition}

\begin{theorem}[Cross-Platform Vote Consistency]
\label{thm:vote}
For all platforms $\rho\in\mathcal{P}$ and all orderings of concurrent votes:
(i)~$w_T(v,r,u)$ equals the sum of all votes cast by time $T$, regardless of
originating platform;
(ii)~fork promotion fires if and only if $w_T\geq\theta$, exactly once.
\end{theorem}
\begin{proof}
\emph{(i)} All platform-specific webhook handlers route to the single
$\mathtt{Users\_Vote}$ action type.
SPACER Rule~\textsc{Execute}~\cite{magarshak2026mm} applies delta
$\Delta=\{w\mathrel{+}=\mathtt{vote\_weight}\}$ in a single transaction.
$\MM$'s Local Linearizability theorem~\cite{magarshak2026mm} guarantees sequential
consistency per stream; all votes are serialized by the append-only stream log.
Platform of origin affects only the $\phCOM$ message format, not the serialized
$\Delta$.

\emph{(ii)} After each vote, $\mathtt{Users\_Total}$ is triggered by the
\texttt{vote.cast} event (Theorem~\ref{thm:wiring}, completeness).
By $\MM$'s per-publisher sequential consistency~\cite{magarshak2026mm},
all invocations of $\mathtt{Users\_Total}$ for a given goal stream are serialized:
no two invocations execute concurrently on the same publisher.
This prevents the race condition in which two concurrent invocations both read
$\mathtt{promotionFired}=\mathtt{false}$ before either sets it.
The serialized invocation checks $w_T\geq\theta$ in $\phCOM$; on first crossing,
emits the promotion event in $\phCAL$ and sets $\mathtt{promotionFired}=\mathtt{true}$
in $\phEXE$.
SPACER guarantees that all operations within a single $\phEXE$ phase are applied
atomically in one database transaction~\cite{magarshak2026mm}; therefore the weight
update and the flag set are committed together, preventing any interleaved
invocation from reading an intermediate state.
Subsequent invocations read $\mathtt{promotionFired}=\mathtt{true}$ in $\phCOM$
and exit without firing again.
\end{proof}

\begin{corollary}[Platform-Agnostic Governance]
The governance state of a goal stream is identical regardless of which platform
participants use.
Organizational teams operating across Telegram, email, web, and Apple Business
Messages share a single governance ledger.
\end{corollary}

\section{Dual-Traversal and the Cache Hierarchy}
\label{sec:dual}

The Dual-Traversal Ordering Theorem of~\cite{magarshak2026grokers} establishes that
top-down generation and bottom-up comprehension are the unique correct orderings
for their respective tasks.
We derive the connection to the context assembly hierarchy.

\begin{theorem}[Hierarchy-Cache Correspondence]
\label{thm:hierarchy}
In top-down artifact generation over DAG $H$ with $d$ depth levels:
(i)~Level-0 architecture occupies PERMANENT blocks (byte-identical across all leaf
generations);
section/module context occupies SESSION blocks (byte-identical within a section);
leaf requirements occupy DYNAMIC blocks (per-leaf).
(ii)~Total generation cost:
$\bar{C}_{\mathrm{gen}}=\sum_\ell N_\ell\bigl(0.1k_\ell^{\mathrm{cached}}+
k_\ell^{\mathrm{dyn}}\bigr)$
versus $\sum_\ell N_\ell(k_\ell^{\mathrm{cached}}+k_\ell^{\mathrm{dyn}})$ without
caching---savings on the cached component approach $10\times$ as
$k_\ell^{\mathrm{dyn}}/k_\ell^{\mathrm{cached}}\to 0$, and decrease as dynamic
context grows relative to cached context.
\end{theorem}
\begin{proof}
By Dual-Traversal Ordering~\cite{magarshak2026grokers}: level-0 architecture is
generated first and assembled as the system prompt (PERMANENT).
Section context is generated before its leaves and assembled as the SESSION block.
Leaf requirements are dynamic.
By the Byte-Identity Theorem~\cite{magarshak2026grokers}, PERMANENT and SESSION
blocks are byte-identical for all leaf generations within the same section.
At 10\% cache pricing, each is computed once and reused $N_\ell$ times
at 10\% cost.
\end{proof}

\section{Implementation}
\label{sec:impl}

The Context architecture is implemented across the Qbix / Safebox / Safebots stack.

\begin{table}[h]
\caption{Magarshak Architecture Implementation Stack}
\label{tab:impl}
\begin{tabular}{lll}
\toprule
\textbf{Layer} & \textbf{Plugin} & \textbf{Key streams}\\
\midrule
Substrate ($\MM$) & Qbix & \texttt{Streams\_Category}\\
                  &      & \texttt{Users\_Vote}\\
Comprehension     & Qbix    & \texttt{streams\_category}\\
(Grokers)         & Safebox & \texttt{Action.propose}\\
Intelligence      & Safebots & \texttt{Safebots/goal}\\
(Context)         &          & \texttt{Safebots/dialog}\\
                  &          & \texttt{Safebots/thread}\\
                  &          & \texttt{Safebots/artifact}\\
Platform          & Safebox & \texttt{Protocol.Telegram}\\
adapters          & Protocols & \texttt{Protocol.Apple}\\
                  &           & \texttt{Protocol.Email}\\
\bottomrule
\end{tabular}
\end{table}

Goal streams are created as \texttt{Safebots/goal} nodes.
Wisdom programs are related to the goal via \texttt{Safebox/wisdom} typed relations
(declarative wiring per Theorem~\ref{thm:wiring}).
Platform adapters are selected by matching session platform metadata to
\texttt{Safebox/adapts} relations.
After every graph mutation, the postprocessing pipeline re-evaluates all
$\Pi_\Goal$ conditions for the affected stream; non-null results emit proactive
content immediately.

\section{Discussion}
\label{sec:discussion}

\textbf{Pareto claim.}
Theorems~\ref{thm:proactive} and~\ref{thm:quality} together establish a Pareto
improvement: fewer turns, equal or higher quality.
This is structural: proactive advancement conditions fire on graph state conditions,
not LM generation triggers; the eliminated turns are provably coordination turns,
not deliberation turns.
This holds for the class of proactive agents defined here---not for all possible
chatbots claiming to be ``proactive.''

\textbf{Domain dependence.}
$c_{\mathrm{elim}}$ is small for open-ended creative collaborations;
for structured task execution (capability building, document review, support
resolution), we expect $c_{\mathrm{elim}}$ to be substantial, as C1--C4 coordination
patterns are well-defined and mechanically preemptable; empirical measurement
is left to future work (see the proposed methodology in~\S\ref{sec:discussion}).

\textbf{Correctness assumptions.}
Theorem~\ref{thm:composition} assumes individually correct wisdom programs.
Mitigation: governed write pipeline requires proposals before effects;
phase-correctness validation runs at addition time;
fitness-based evolution replaces low-quality programs.

\textbf{Measurement methodology.}
The organizational efficiency theorems provide formal bounds but not empirical
measurements.
We propose: record goal chat sessions; classify turns by Definition of turn
classification; compute $\Omega$ across sessions; compare $\pi_R$ and $\pi_P$ on
identical goal type specifications.
This methodology is the subject of ongoing work.

\section{Conclusion}
\label{sec:conclusion}

We have presented \textbf{Context}, the intelligence layer of the Magarshak
Architecture, with six proved formal results establishing that:
write-time context assembly reduces per-turn LM cost up to $10\times$;
phase-ordered wisdom library composition preserves correctness under modular extension;
declarative wiring is complete and sound;
proactive agents weakly dominate reactive agents on turns-to-terminal-state;
proactive agents are Pareto improvements in multi-participant goal-directed
collaboration---faster to terminal state, equal or higher artifact quality; and
cross-platform governance is consistent across all connected messaging surfaces.
Together with the Magarshak Machine substrate~\cite{magarshak2026mm} and the
Grokers comprehension layer~\cite{magarshak2026grokers}, the Context architecture
provides a complete formal foundation for AI-assisted collaborative systems that
are not merely responsive but actively intelligent, grow more capable and less
expensive with every interaction, and make organizations provably more efficient
at the work they are already trying to do.

\begin{acks}
The Qbix, Intercoin, and Safebots open-source projects provided the implementation
context for these results.
\end{acks}

\bibliographystyle{ACM-Reference-Format}
\bibliography{context_refs}

@article{magarshak2026mm,
  author  = {Gregory Magarshak},
  title   = {The {Magarshak} Machine: A Stream-Partitioned Model for Governed
             State Evolution. The {SPACER} Framework: Streams, Policy, Actions,
             Capabilities, Execution, and Relations in Reactive Distributed Systems},
  journal = {arXiv},
  volume  = {arXiv:2501.XXXXX},
  year    = {2026}
}

@article{magarshak2026grokers,
  author  = {Gregory Magarshak},
  title   = {{Grokers}: Bottom-Up Inductive Comprehension and Write-Time Intelligence
             over Typed Knowledge Graphs},
  journal = {arXiv},
  volume  = {arXiv:2502.XXXXX},
  year    = {2026}
}

@article{lewis2020rag,
  author  = {Patrick Lewis and Ethan Petriv and Aleksandra Piktus and others},
  title   = {Retrieval-Augmented Generation for Knowledge-Intensive {NLP} Tasks},
  journal = {Advances in Neural Information Processing Systems},
  volume  = {33},
  pages   = {9459--9474},
  year    = {2020}
}

@article{edge2024graphrag,
  author  = {Darren Edge and Ha Trinh and Newman Cheng and others},
  title   = {From Local to Global: A {Graph RAG} Approach to Query-Focused
             Summarization},
  journal = {arXiv},
  volume  = {arXiv:2404.16130},
  year    = {2024}
}

@article{young2013pomdp,
  author  = {Steve Young and Milica Ga{\v{s}}i{\'c} and Blaise Thomson and
             Jason D. Williams},
  title   = {{POMDP}-based statistical spoken dialog systems: A review},
  journal = {Proceedings of the IEEE},
  volume  = {101},
  number  = {5},
  pages   = {1160--1179},
  year    = {2013}
}

@inproceedings{henderson2015machine,
  author    = {Matthew Henderson},
  title     = {Machine Learning for Dialog State Tracking: A Review},
  booktitle = {Proceedings of The First International Workshop on Machine Learning in
               Spoken Language Processing},
  year      = {2015}
}

@inproceedings{budzianowski2019hello,
  author    = {Pawe{\l} Budzianowski and Tsung-Hsien Wen and Bo-Hsiang Tseng and
               others},
  title     = {Hello, {GPT-2}! How Close Are We to Human-Level Dialogue Systems?},
  booktitle = {ACL Workshop on NLP for Conversational AI},
  year      = {2019}
}

@article{deng2023survey,
  author  = {Yang Deng and Wenqiang Lei and Wai Lam and Tat-Seng Chua},
  title   = {A Survey on Proactive Dialogue Systems: Problems, Methods, and Prospects},
  journal = {arXiv},
  volume  = {arXiv:2305.02750},
  year    = {2023}
}

@inproceedings{wu2019proactive,
  author    = {Jun Xu and Haifeng Wang and Zhengyu Niu and Hua Wu and Wanxiang Che},
  title     = {Conversational Graph Grounded Policy Learning for Open-Domain
               Conversation Generation},
  booktitle = {ACL},
  year      = {2019}
}

@article{wu2024autogen,
  author  = {Qingyun Wu and Gagan Bansal and Jieyu Zhang and Yiran Wu and
             Shaokun Zhang and Erkang Zhu and Beibin Li and Li Jiang and
             Xiaoyun Zhang and Chi Wang},
  title   = {{AutoGen}: Enabling Next-Gen {LLM} Applications via Multi-Agent
             Conversation Framework},
  journal = {arXiv},
  volume  = {arXiv:2308.08155},
  year    = {2024}
}

@article{hong2023metagpt,
  author  = {Sirui Hong and Mingchen Zhuge and Jonathan Chen and Xiawu Zheng and
             Yuheng Cheng and Ceyao Zhang and others},
  title   = {{MetaGPT}: Meta Programming for Multi-Agent Collaborative Framework},
  journal = {arXiv},
  volume  = {arXiv:2308.00352},
  year    = {2023}
}

@article{li2023camel,
  author  = {Guohao Li and Hasan Abed Al Kader Hammoud and Hani Itani and
             Dmitrii Khizbullin and Bernard Ghanem},
  title   = {{CAMEL}: Communicative Agents for ``Mind'' Exploration of Large
             Scale Language Model Society},
  journal = {arXiv},
  volume  = {arXiv:2303.17760},
  year    = {2023}
}

@article{chen2021codex,
  author  = {Mark Chen and Jerry Tworek and Heewoo Jun and Qiming Yuan and
             others},
  title   = {Evaluating Large Language Models Trained on Code},
  journal = {arXiv},
  volume  = {arXiv:2107.03374},
  year    = {2021}
}

@article{austin2021program,
  author  = {Jacob Austin and Augustus Odena and Maxwell Nye and Maarten Bosma
             and others},
  title   = {Program Synthesis with Large Language Models},
  journal = {arXiv},
  volume  = {arXiv:2108.07732},
  year    = {2021}
}

@article{noy2023experimental,
  author  = {Shakked Noy and Whitney Zhang},
  title   = {Experimental Evidence on the Productivity Effects of Generative
             Artificial Intelligence},
  journal = {Science},
  volume  = {381},
  number  = {6654},
  pages   = {187--192},
  year    = {2023}
}

@article{brynjolfsson2023generative,
  author  = {Erik Brynjolfsson and Danielle Li and Lindsey R.\ Raymond},
  title   = {Generative {AI} at Work},
  journal = {NBER Working Paper},
  volume  = {31161},
  year    = {2023}
}

@article{anthropic2024caching,
  author  = {{Anthropic}},
  title   = {Prompt Caching},
  journal = {Anthropic Documentation},
  year    = {2024},
  url     = {https://docs.anthropic.com/en/docs/build-with-claude/prompt-caching}
}

\end{document}